\documentclass[conference]{IEEEtran}
\IEEEoverridecommandlockouts
\usepackage{cite}
\usepackage{amsmath,amssymb,amsfonts}
\usepackage{graphicx}
\usepackage{xcolor}
\usepackage{multirow}
\usepackage{amsthm}
\usepackage{statmath}
\usepackage{epsfig}
\usepackage{pdfpages}
\usepackage{bm} 
\usepackage{hyperref}
\usepackage{siunitx}
\usepackage{multirow, booktabs}

\newcolumntype{P}[1]{>{\centering\arraybackslash}p{#1}}
\theoremstyle{definition}
\newtheorem{definition}{Definition}
\usepackage{algpseudocode}
\usepackage{algorithm}
\newtheorem{prop}{Proposition}
\algnewcommand\algorithmicforeach{\textbf{for each}}
\algdef{S}[FOR]{ForEach}[1]{\algorithmicforeach\ #1\ \algorithmicdo}
\newtheorem{theorem}{Theorem}
\newtheorem{lemma}[theorem]{Lemma}
\def\BibTeX{{\rm B\kern-.05em{\sc i\kern-.025em b}\kern-.08em
    T\kern-.1667em\lower.7ex\hbox{E}\kern-.125emX}}
\algdef{S}[FOR]{ForEach}[1]{\algorithmicforeach\ #1\ \algorithmicdo}
\def\BibTeX{{\rm B\kern-.05em{\sc i\kern-.025em b}\kern-.08em
    T\kern-.1667em\lower.7ex\hbox{E}\kern-.125emX}}

\begin{document}
\bstctlcite{IEEEexample:BSTcontrol}
\author{\IEEEauthorblockN{Zhiyuan Li\IEEEauthorrefmark{1}\thanks{Copyright (c) 2023 IEEE. Personal use of this material is permitted.
Permission from IEEE must be obtained for all other uses, in any current or
future media, including reprinting/republishing this material for advertising or
promotional purposes, creating new collective works, for resale or redistribution to servers or lists, or reuse of any copyrighted component of this work
in other works.},~\IEEEmembership{Student Member,~IEEE},
            Ziru Liu\IEEEauthorrefmark{2}, \thanks{Contact authors: li3z3@mail.uc.edu (Zhiyuan Li); ralescal@ucmail.uc.edu (Anca Ralescu).}
            Anna Zou\IEEEauthorrefmark{3}, 
            Anca L. Ralescu\IEEEauthorrefmark{1},~\IEEEmembership{Senior Member,~IEEE}}
        \IEEEauthorblockA{\IEEEauthorrefmark{1}\emph{Department of Computer Science, University of Cincinnati}}
        \IEEEauthorblockA{\IEEEauthorrefmark{2}\emph{School of Data Science, City University of Hong Kong}}
        \IEEEauthorblockA{\IEEEauthorrefmark{3}\emph{Directorate for Social, Behavioral and Economic Sciences, National Science Foundation}}
}
\title{Learning Empirical Bregman Divergence for Uncertain Distance Representation}
\maketitle
\begin{abstract}
Deep metric learning techniques have been used for visual representation in various supervised and unsupervised learning tasks through learning embeddings of samples with deep networks. However, classic approaches, which employ a fixed distance metric as a similarity function between two embeddings, may lead to suboptimal performance for capturing the complex data distribution. The Bregman divergence generalizes measures of various distance metrics and arises throughout many fields of deep metric learning. In this paper, we first show how deep metric learning loss can arise from the Bregman divergence. We then introduce a novel method for learning empirical Bregman divergence directly from data based on parameterizing the convex function underlying the Bregman divergence with a deep learning setting. We further experimentally show that our approach performs effectively on five popular public datasets compared to other SOTA deep metric learning methods, particularly for pattern recognition problems.
\end{abstract}

\section{Introduction}
Deep metric learning formulates a task-specific problem for learning the distance metrics among samples. The learned distance can then be applied to object detection, matching, ranking, and other machine learning tasks \cite{mees2017metric,meyer2019importance,rezayati2022improving}. Despite the success and advances of deep metric learning techniques across many applications, selecting the optimal distance metric, or even a general distance metric remains an uncertain task. Yet the chosen distance metric of the learning loss function can be a key factor in deciding the performance of deep learning models by learning the feature representations within the geometric or probabilistic space \cite{bellet2015metric}. In contrast to previous works, instead of selecting the uncertain distance metric, our goal is to learn a generalized distance metric for deep learning classification using the Bregman divergence \cite{bregman1967relaxation}. 

Classic methods of deep metric learning \cite{koch2015siamese,hoffer2015deep,khosla2020supervised} aim to learn a robust feature representation by training a deep encoder over the input space to maximize the distance between similar samples (positive pairs) and minimize the distance between dissimilar ones (negative pairs). For example, the Siamese network \cite{koch2015siamese} uses the Euclidean distance to calculate the distance metric between two feature embeddings and employs a Softmax function to convert the computed distance into a probability to present the similarity score. The Triplet network \cite{hoffer2015deep}, an extended version of the Siamese network, takes a triplet input (anchor, negative and positive) to the deep encoder and aims to group together anchor and positive samples and push away anchor and negative samples. SupCon \cite{khosla2020supervised}, a supervised contrastive learning approach, utilizes cosine similarity as a distance metric in the loss function for learning the discriminative features among samples for classification. Contrastive learning methods, such as SupCon \cite{khosla2020supervised}, Invariant \cite{ye2019unsupervised}, etc., use a similar idea to deep metric learning but with a self-supervision approach. These works can be summarized in two steps: 1) train an embedding function (neural network) $f_{\theta}(\cdot)$ to learn the similarity features among samples with the same labels, 2) fine-tune the pre-trained embedding function $\hat{f}_{\theta}(\cdot)$ for a supervised classification task. Although these methods have achieved excellent results in deep learning classification, the distance metrics are manually selected, which may lead to a sub-optimal performance during training. 

More recently, integrating deep metric learning and statistical divergence has achieved popularity in the metric learning area. One of the most famous statistical divergences, the Bregman divergence, is generated by a strictly convex and continuously differentiable function $\phi$ defined on a closed convex set \cite{bregman1967relaxation}. Depending on the selected underlying convex function, specific distance metrics, such as Euclidean or cosine similarity, can be generated. For example, the Bregman divergence can be chosen as the well-known Kullback-Leibler (KL)-divergence to measure the probabilistic distance between two inputs where comparison is needed for the distributions. However, the learning objective still remains uncertain since the standard family of the Bregman divergences may not fully capture the patterns of samples. In deep divergence learning, employing the Bregman divergence as a deep learning setting captures the nonlinear relationship for learning more generalizable distance metrics among samples \cite{cilingir2020deep}. For example, Siahkamari et al. \cite{siahkamari2020learning} used arbitrary Bregman divergence to learn the underlying divergence function through a piecewise linear approximation approach. Cilingir et al. \cite{cilingir2020deep} proposed deep Bregman divergences by formulating the metric learning task into a particular case of symmetric divergences. However, these works use a linear max-affine function to parameterize $\phi$ while the nonlinear properties are ignored, and the gradient of the learning loss may vanish during the training. 

This paper introduces a framework to learn the functional arbitrary Bregman divergence for distance representation with deep metric and contrastive learning styles, which can be applied to foundational visual representation tasks. We first investigate the relationship between defined metric learning loss and the Bregman divergence. In this setting, we show that for any probability-based similarity measurements using the Softmax function, their metric learning loss can be seen to arise from the Bregman divergence. These included deep metric learning models, i.e., Siamese network, Triplet network, supervised contrastive learning, and a typical contrastive learning model, such as SupCon \cite{khosla2020supervised}. We then turn our attention to learning a Bregman divergence directly through the deep learning approach. In contrast to previous works on pre-fixed distance metrics, e.g., the Euclidean distance and cosine similarity, the learned Bregman divergence represents a \emph{generalizable} solution to effectively capture the similarity information between samples for various deep metric learning tasks. To achieve this, we used the generalized nonlinear models (GNMs) to smoothly parameterize the strictly convex and continuously differentiable function $\phi$. Then, a standard deep learning setting is performed to learn the Bregman divergence using the gradient-descent-based algorithm. To evaluate the performance of our proposed approach, we employed two public datasets to show the empirical results that highlight the effectiveness of our proposed method. In particular, we showed that learning a Bregman divergence directly through the deep learning setting offers classification performance gains over learning a fixed distance metric. We also showed that the learned distance metrics could capture more complex data distribution than other state-of-the-art (SOTA) methods. 

Our main contributions of this work are summarized as follows: (1) We consider the implicit relationship between deep metric learning and the Bregman divergence by proving that for any defined metric learning loss, the general distance metric form can arise from the Bregman divergence. (2) We present a novel framework to learn the uncertain Bregman divergence in a deep learning setting. Instead of fixing the distance metric function during training, we employ the GNMs to parameterize the generating function $\phi$ in the Bregman divergence. Our approach, strictly convex and smoothness, approximates $\phi$ arbitrarily well. (3) With theoretical analysis and extensive experiments, our approach demonstrates the effectiveness of learned empirical distance representation over other SOTA methods in deep metric and contrastive learning settings.

\section{Related Work}
In this section, we first provide a brief overview of deep metric learning and then discuss the related works of the Bregman divergence learning.

\textbf{Deep Metric Learning}: 
With the popularity of deep learning techniques, researchers have started to perform metric learning tasks in a deep learning setting \cite{kaya2019deep}. Similar to classic metric learning methods, deep metric learning focuses on learning the similarity relationship among samples using deep features (e.g., feature embeddings). In this setting,  suppose $f_{\theta}(\cdot)$ is a function that embeds the sample $x$ into a feature embedding $f_{\theta}(x)$, a gradient-descent-based optimization algorithm is employed to iteratively learn the nonlinear distance metric among embedded features $f_{\theta}(x)$ and $f_{\theta}(y)$. Several well-defined loss functions, such as contrastive loss \cite{chopra2005learning}, triplet loss \cite{hoffer2015deep}, NCE loss \cite{gutmann2010noise}, and SupCon loss \cite{khosla2020supervised}, have been proposed to learn discriminative features for classification. However, all of these loss functions use a fixed distance, either Euclidean distance $\|f_{\theta}(x)-\emph{f}_{\theta}(y)\|^2$ or dot product $f_{\theta}(x)^T\emph{f}_{\theta}(y)$.

\textbf{Bregman Divergence Learning}:
Another approach that goes beyond linear metric learning is the Bregman divergence learning framework. The common idea is to generalize the learning distance metric into an arbitrary Bregman divergence \cite{banerjee2005clustering}. Here, beyond the linear metric, in the Bregman divergence, more general asymmetric divergence, such as the KL-divergence, Itakura-Satio divergence, and others, are also considered as the nonlinear distance metric, resulting in more robust performance than linear methods. Learning the functional Bregman divergence is being explored to extend the standard Bregman divergence into a more generalizable form. Instead of taking two vectors as input in functional Bregman divergence, here, we compute the divergence between two functions or probability distributions \cite{frigyik2008functional}. The existing works of learning functional Bregman divergence can be divided into two directions: (1) integrate contrastive learning, and Bregman divergence \cite{rezaei2021deep,lu2022neural}, (2) parameterize the generating function $\phi$ of the Bregman divergence using the max-affine representations approach \cite{siahkamari2020learning, cilingir2020deep}. However, these works have the following shortcomings: (1) the implicit connection between contrastive loss and the Bregman divergence is ignored, (2) the piecewise linear approximation does not yield the continuously differentiable (smoothness) property. In our approach, not only do we study the inner connection between contrastive loss and the Bregman divergence but also directly learn the generation function $\phi$ directly using a set of smooth GNMs approach. 

\section{Learning Bregman Divergence}
In this section, we first formally introduce some definitions and background that will be used throughout the rest of the paper.  We then turn out attention to learning the Bregman divergences for a deep metric learning task, the main contribution of our work.

\subsection{Preliminaries}
\textbf{Bregman Divergence}. The Bregman divergence \cite{bregman1967relaxation} represents a general distance metric between two data inputs. Let $\phi$ be a strictly convex and continuously differentiable function defined on a closed convex set $\Omega \in \mathbb R^d$. The Bregman divergence between two inputs $x$ and $y$ are defined as 
\begin{align}
    d_{\phi}(x,y)=\phi(x)-\phi(y)-(x-y)^T\nabla \phi(y)
\end{align}
where $\nabla \phi(y)$ is the first-order derivative of $\phi(y)$. Examples of several well-known 
distance metrics such as Euclidean distance, KL-divergence, and Itakura-Satio divergence can be parameterized to the Bregman divergence form of \textbf{Eq (1)}. In this paper, we consider the extended version of the Bregman divergence via \emph{functional Bregman divergences}.

\textbf{Functional Bregman Divergence}. Similar to classic Bregman divergences, a functional Bregman divergence \cite{frigyik2008functional} measures the distance between two functions (e.g., probability distributions). Given two functions $p$ and $q$, and a strictly convex function $\phi$, the corresponding functional Bregman divergence is defined by
\begin{align}
    d_{\phi}(p,q)=\phi(p)-\phi(q)-\int \left[p(x)-\emph{q}(x)\right]\delta\phi(q)(x)dx
\end{align}
where $\delta\phi(q)$ represents the functional derivative of $\phi$ at $q$. Same as the classic Bregman divergence, the functional Bregman divergences hold the same properties, such as convexity, non-negativity, linearity, and others. 

\textbf{Convexity}. The Bregman divergence, $\phi$ is restricted to be strictly convex, which constrains the parameterization of the Bregman divergence when choosing a $\phi$. In this paper, we learn the arbitrary Bregman divergence directly through a deep learning approach and consider learning an \emph{optimal} $\phi$. To approach this, we recall the definition of strict convexity:

\begin{definition}
A function $f:\mathbb R^n\to \mathbb R$ is strictly convex if 
    \begin{align}
        f(\lambda x + (1-\lambda)y)<\lambda f(x)+(1-\lambda)f(y)
    \end{align}
 where $\forall x,y,x\neq y, \; \forall \lambda\in (0,1)$.
\end{definition}
\begin{figure*}
    \centering
    \includegraphics[width=0.88\textwidth]{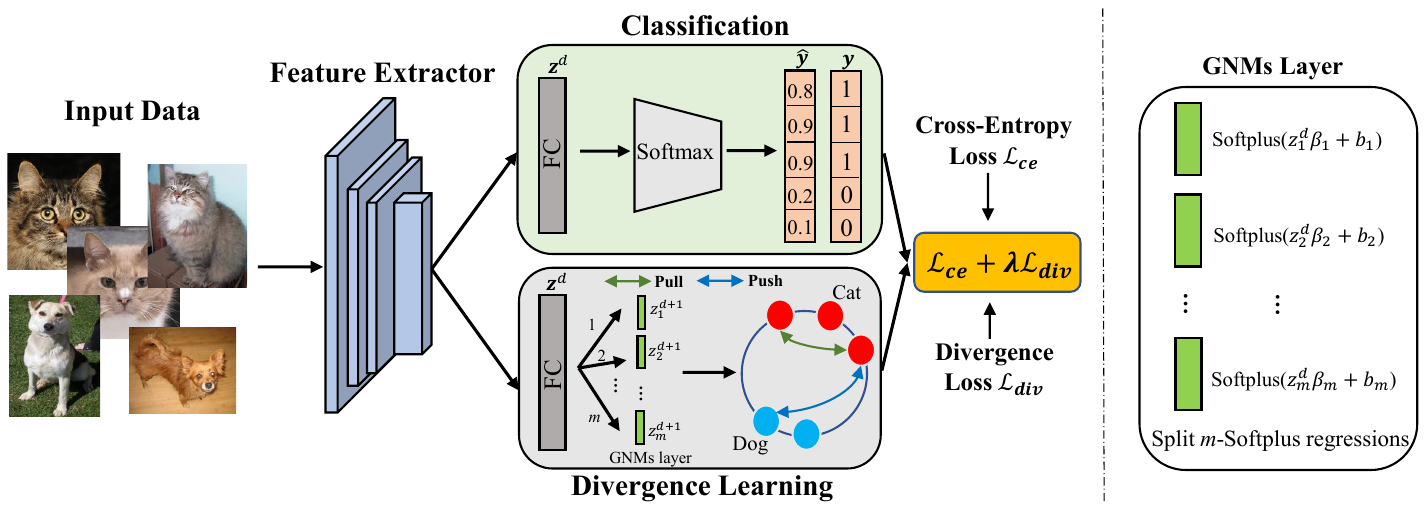}
    \caption{The overview of our proposed framework. We use a pre-trained ResNet18 as the encoder and learn joint tasks of supervised classification and Bregman divergence learning. We employed a group of generalized nonlinear models (GNMs) with the Softplus function to represent $\phi$. The learned distance representation will be further used in image classification tasks using a $k$NN classifier. (Top branch applies cross-entropy loss for classification while bottom branch applies divergence loss for distance learning.)}
\end{figure*}

\subsection{Bregman Divergence View of Deep Metric Learning}
Let $D=\{x_i,y_i\}^{N}_{i=1}$ denote the training data, where $x_i$ is the sample, and its corresponding label $y_i$, $S(i)\in D$ denotes the set of indices for positive pair samples, i.e., $S(i):\{j\in \Lambda|y_i=y_j, i\neq j\}$. Similarly, $K(i)\in D:\{j\in \Lambda|y_i\neq y_j, i\neq j\}$ denotes the indices set for negative pair samples. The probability to recognize  $x_i,x_s,s\in S(i)$ as $y_i$ can be formulated as
\begin{align}
    p(y_i|x_i,x_s)=\frac{\exp(z^T_iz_s)}{\sum_{j\in \Lambda,j\neq i}\exp(z^T_iz_j)}
\end{align}
where $z_i,z_j\in \mathbb R^d$ denotes the $(i,j)_\text{th}$ embeddings extracted from an encoding function $f_{\theta}()$, such that, $z_i=f_{\theta}(x_i)$. Likewise, the probability of $x_i,x_k,k\in K(i)$ is being recognized as $y_i$ can be formulated as
\begin{align}
    p(y_i|x_i,x_k)=\frac{\exp(z^T_iz_k)}{\sum_{j\in \Lambda,j\neq i}\exp(z^T_iz_j)}
\end{align}
To learn a representative distance metric, which groups positive pairs and pushes away the negative pairs, we need to maximize $p(y_i|x_i,x_s)$ and minimize $p(y_i|x_i,x_k)$, simultaneously. Thus, the objective function leads to a maximum likelihood estimation, which is 
\begin{align}
    \ell_i=\prod_{s\in S(i)}\prod_{k\in K(i)}p(y_i|x_i,x_s)\left[1-p(y_i|x_i,x_k)\right]
\end{align}
Thereby, learning loss is the negative-log-likelihood of $\ell_i$ over all the data points indexed by $\Lambda$, which simplifies $\ell_i$ to 
\begin{align}
    \mathcal{L}=-\sum_{i\in \Lambda}&\|\bm{K}(i)\|\sum_{s\in \bm{S}(i)}\log p(y_i|x_i,x_s) \nonumber \\
     &-\sum_{i\in \Lambda}\|\bm{S}(i)\|\sum_{k\in \bm{K}(i)}\log\left[ 1-p(y_i|x_i,x_k)\right]
\end{align}
where $\|\bm{S}(i)\|$ and $\|\bm{K}(i)\|$ are the sizes of their corresponding set.

\begin{prop}
For any probability-based distance metric between two inputs, i.e., $d(x,y)$, with the Softmax function, there exists a general distance form of $d(x,y)$, which arises from the Bregman divergence.
\end{prop} 

\begin{proof}
Here, we prove the \textbf{Proposition 1} starting from the above \textbf{Eq(7)}. For $-\log p(y_i|x_i,x_s)$, based on the first-order Taylor approximation, we have
\begin{align}
    -\log p(y_i|x_i,x_s)&=\log\frac{\sum_{j\in \Lambda,j\neq i}\exp(z^T_iz_j)}{\exp(z^T_iz_s)} \nonumber \\
    &=\log \left(1+\sum_{j\neq i,s}\exp\left[z^T_iz_j-z^T_iz_s\right]\right)  \nonumber \\
    &\approx \sum_{j\neq i,s}\exp\left(z^T_iz_j-z^T_iz_s\right)  \nonumber \\
    &\approx 1+\sum_{j\neq i,s}z^T_iz_j-z^T_iz_s 
\end{align}
Notice that, $z^T_iz_j=1-\frac{1}{2}\|z_i-z_j\|^2$ if $\|z\|=1$. With $\phi(x)=\|x\|^2$, the general distance of $\log p(y_i|x_i,x_s)$ arises from the functional Bregman divergence is 
\begin{align}
    \log p(y_i|x_i,x_s)&\approx \sum_{i\neq s}d_{\phi}\left(f_{\theta}(x_i),f_{\theta}(x_s)\right)-\sum_{j\neq i}d_{\phi}\left(f_{\theta}(x_i),f_{\theta}(x_j)\right) 
 \nonumber \\
    &= \Theta\left(f_{\theta}(x_i),f_{\theta}(x_s),f_{\theta}(x_j) \right) 
\end{align}
Next, we turn our attention to $p(y_i|x_i,x_k)$ of \textbf{Eq (5)}. Following the same approach of \textbf{Eq (8)}, the log of $p(y_i|x_i,x_k)$ can be expressed as
\begin{align}
    \log p(y_i|x_i,x_k)&\approx  \sum_{i\neq k}d_{\phi}\left(f_{\theta}(x_i),f_{\theta}(x_k)\right)-\sum_{j\neq i}d_{\phi}\left(f_{\theta}(x_i),f_{\theta}(x_j)\right) 
    \nonumber \\
    &=\Theta\left(f_{\theta}(x_i),f_{\theta}(x_k),f_{\theta}(x_j) \right) 
\end{align}
Thus, we formulate the learning loss $\mathcal{L}$ of \textbf{Eq (7)} as the Bregman divergences approximation without any scale parameters:
\begin{align}
   \mathcal{L}_{div}\approx-\sum_{i\in \Lambda}\|&\bm{K}(i)\|\sum_{s\in \bm{S}(i)}\Theta\left(f_{\theta}(x_i),f_{\theta}(x_s),f_{\theta}(x_j) \right)  \nonumber \\
    &+\sum_{i\in \Lambda}\|\bm{S}(i)\|\sum_{k\in \bm{K}(i)}\Theta\left(f_{\theta}(x_i),f_{\theta}(x_k),f_{\theta}(x_j) \right) 
\end{align}
This yields the initial idea of divergence learning that maximizes the functional Bregman divergence among a positive pair $\Theta\left(f_{\theta}(x_i),f_{\theta}(x_s),f_{\theta}(x_j) \right) $ and minimizes $\Theta\left(f_{\theta}(x_i),f_{\theta}(x_k),f_{\theta}(x_j) \right)$, the functional Bregman divergence among a negative pair.
\end{proof}

\subsection{Deep Bregman Divergence Learning With GNMs}
\textbf{Parameterization}. We consider learning a functional Bregman divergence $d_{\phi}$ with the deep learning setting. Suppose $\bm{z}^{d}$ is the $d_{\text{th}}$ layer output embedding vector, and $\bm{z}^{d+1}$ is the target embedding, which will be the input for functional Bregman divergence $d_{\phi}$. To directly learn $d_{\phi}$, one thing that needs to be considered is the convexity property of $\phi$. We employ a set of \emph{generalized nonlinear models} to estimate each value of $\bm{z}^{d+1}$. In this case, all $z_i^{d+1} \in \bm{z}^{d+1}$ are independent to each other. Let $\bm{\beta},\bm{b}$ denote the weights and biases, respectively, the $i_{\text{th}}$ expectation value of the $\bm{z}^{d+1}$ is 
\begin{align}
    \mathbb E(z_i^{d+1}|z_i^{d})=\alpha(z_i^{d}\beta_i+b_i)
\end{align}
where $\beta_i\in \bm{\beta},b_i\in \bm{b}$, and $\alpha$ is a convex link function. Note that our learning loss is a functional Bregman divergence, in which $\delta\phi(q)(x)$ involves the second derivatives. Similar to \cite{lu2022neural}, we employ a Softplus as the \emph{parametric link function}.
\begin{lemma}
The Softplus function $\alpha(x)=\log(1+\exp(wx))$ is strictly monotonically increasing, strictly convex, and smooth.
\end{lemma}
\begin{proof}
Obviously, the first and second derivative of $\alpha(x)$ is always positive such that
\begin{align}
    \frac{d}{dx}log\left(1+\exp(wx)\right)&=(1+\exp(-wx))^{-1} \in (0,1) \nonumber \\
    \frac{d^2}{dx^2}log\left(1+\exp(wx)\right)&=\frac{w\exp(wx)}{(1+\exp(wx))^2} \in (0,1)
\end{align}
Thus, $\alpha(x)$ is strictly monotonically increasing, strictly convex, and smooth. 
\end{proof}

\begin{prop}
 The expectation $\mathbb E(\bm{z}^{d+1}|\bm{z}^{d})$ for $(d+1)_{\text{th}}$ layer of the embedding outputs is also strictly monotonically increasing, strictly convex, and smooth.   
\end{prop}

\begin{proof}
    Let $\lambda \in (0,1)$, $\alpha: \mathbb R^d\to \mathbb R$ denotes a Softplus function, and $g: \mathbb R^n\to \mathbb R^d$ denotes an affine function, such that $g(\bm{z})=\bm{z}\bm{\beta}+\bm{b}$. Suppose $z_1, z_2\in \bm{z}$, we have 
    \begin{align}
        g(\lambda z_1 + (1-\lambda) z_2) = \lambda g(z_1) + (1-\lambda)g(z_2) 
    \end{align}
    By \textbf{Definition 1} and \textbf{Lemma 1}, $\mathbb E=\alpha\circ g$ is as:
    \begin{align}
        \mathbb E(\lambda z_1 + (1-\lambda)z_2) &= \lambda \alpha\left(g(z_1)+(1-\lambda) g(z_2)\right) \nonumber \\
        &<\lambda \alpha \circ g(z_1) + (1-\lambda) \alpha \circ g(z_2) \nonumber \\
        &= \lambda \mathbb E(z_1) + (1-\lambda)\mathbb E(z_2)
    \end{align}
    which means that $\mathbb E$ a strictly convex function. We know that $\mathbb E(\bm{z})=\log\left(1+\exp(\bm{z}\bm{\beta}+\bm{b})\right)$, and $\frac{\partial}{\partial \bm{z}}E(\bm{z}), \frac{\partial^2}{\partial \bm{z}^2}E(\bm{z})$, which are always positive. Thus, $\mathbb E$ is strictly monotonically increasing and smooth.
\end{proof}

\textbf{Learning Algorithm}. With the above foundations, we propose a training algorithm to learn the arbitrary functional Bregman divergence $d_{\phi}$ for classification in \textbf{Figure 1}. Our approach is two-fold: (1) a metric learning task for learning $d_{\phi}$ among samples from an input batch, (2) a classification task for learning the label information. Specifically, we directly employ the defined divergence loss $\mathcal{L}_{div}$ in \textbf{Eq (11)} as distance metrics and train it with a cross-entropy loss $\mathcal{L}_{ce}$ jointly. In this case, the learned distance metrics capture more patterns than any pre-fixed distance metrics, leading to higher predictive power for classification. Let $D=(x_i,y_i)_{i=1}^{N}$ denote the training dataset with labels $\bm{y}$, a positive pairs set $D_s=(x_i,x_s)_{i=1}^{N}, y_i=y_s$, negative pairs set  $D_k=(x_i,x_k)_{i=1}^{N}, y_i\neq y_k$, a arbitrary deep encoder $f_{\theta}$, where the pseudo-code is summarized in \textbf{Algorithm 1}.
\begin{algorithm}
\caption{\textbf{Deep Bregman Divergence Learning for Classification Via Joint Training}}
\begin{algorithmic}
\Require $D, D_s,D_k,\mathcal{L}_{div},f_{\theta}$, and $\mathcal{L}_{ce}$.
\State $f^{d}_{\theta} \gets$ a $d_{\text{th}}$ layer of $f_{\theta}$ for feature extraction
\ForEach {$(x_i,x_j,y_i)\in D,(x_i,x_s)\in D_s,(x_ix_k)\in D_k$}
\State $z_i^{d}, z_j^{d}, z_s^{d}, z_k^{d}\gets f_{\theta}^{d}(x_i), f_{\theta}^{d}(x_j), f_{\theta}^{d}
(x_s), f_{\theta}^{d}(x_k)$
\State Compute $\phi(z_i^d),\phi(z_j^d),\phi(z_i^s),\phi(z_j^k)$ \Comment{\textbf{Eq (12)}}
\State $d_{\phi_1},d_{\phi_2},d_{\phi_3}\gets d_{\phi}(z_i^d,z_j^d),d_{\phi}(z_i^d,z_s^d),d_{\phi}(z_i^d,z_k^d)$
\State $\ell_{div}\gets \mathcal{L}_{div}(d_{\phi_1},d_{\phi_2},d_{\phi_3})$ \Comment{\textbf{Eq (11)}}
\State $\ell_{ce}\gets  \mathcal{L}_{ce}(x_i,y_i)$ \Comment{cross-entropy loss}
\State $\mathcal{L}^{*}\gets \ell_{ce}+\gamma \ell_{div}$ \Comment{joint training}
\State $\mathcal{L}^{*}.\text{backward()}$ \Comment{perform backpropagation}
\EndFor
\State \Return $\hat{f}_{\theta}$ \Comment{the pre-trained $f_{\theta}$}
\end{algorithmic}
\end{algorithm}

\section{Data and Experiments}
\subsection{Dataset}
We employed five datasets for image recognition, namely iChallenge-PM \cite{fu2019palm}, iChallenge-AMD \cite{fang2022adam}, Caltech-UCSD Birds (CUB200 dataset) \cite{welinder2010caltech}, Animal FaceHQ (AFHQ) \cite{choi2020stargan} and Oxford-III Pet \cite{parkhi2012cats}. The dataset consists of 1,200 annotated retinal fundus images from 2 classes, 1,200 color fundus images with 400 ones released with annotations from 2 classes, 11,788 bird images from 200 classes, 16,130 animal images from 3 classes, and 7,349 from 2 classes, respectively. To provide the advantage of our approach in small-size samples, we randomly selected 300 and 1000 images from the AFHQ and the Oxford-III pet with an equal ratio of each class, respectively,

\subsection{Implementation Details}
As shown in \textbf{Figure 1}, our approach is built on a network backbone, i.e., pre-trained ResNet18 \cite{he2016deep}, with the same setting in our previous study \cite{li2023learning, li2023novel} for feature extraction. With the output of ResNet18, $\textbf{f}$, is then connected to a Multi-Layer Perceptron (MLP) layer, followed by batch normalization and a Rectified Linear Unit (ReLU) activation function. The output of this process is reduced feature dimension to 128, denoted as $\bm{z}$. For the classification branch, $\bm{z}$ is connected to a Softmax and cross-entropy loss $\mathcal{L}_{ce}$ with labels. For the divergence learning task, $\bm{z}$ is input to a $L_2$ norm layer, resulting in $\|\bm{z}\|=1$, and sequentially followed by a GNMs layer fused with $k$-Softplus regression outputs. Furthermore, a divergence metric loss $\mathcal{L}_{div}$ is employed to learn the arbitrary Bregman divergence. To test the learned Bregman divergence, we applied the $k$NN classifier based on $\bm{z}$, in which we set $k=50$ empirically. We randomly resized each image within a range of 0.3 to 1.0 for each batch size. The batch size was set to 32, and the model optimization was performed using the Adam optimizer. We set the learning rate and weight decay to 0.0001 and trained the whole framework for 2000 epochs. To evaluate the model, we used accuracy and the Area Under the Receiver Operating Characteristic (ROC) curve (AUC). Following standard practice, we used 10-fold cross-validation to evaluate each competing method. In addition, we conducted a non-parametric Wilcoxon test with a significance level of 0.05 for all statistical inferences using R-studio. The framework was implemented using python 3.8, Scikit-Learn 0.24.1, Pytorch 1.9.1, and Cuda 11.1 on a NVIDIA GeForce GTX 1660 SUPER GPU.

\subsection{Competing State-of-The-Art Methods}
We compared our approach with other SOTA methods across the deep metric learning and contrastive learning, including Siamese network \cite{koch2015siamese}, Triplet network \cite{hoffer2015deep}, N-pair \cite{sohn2016improved}, SupCon \cite{khosla2020supervised}, GHM \cite{li2023learning}, PDBL \cite{siahkamari2020learning}, and DeepDiv \cite{cilingir2020deep} using their released code on GitHub. A supervised learning baseline was also included by modifying the last fully connected layer of ResNet18 to match the number of classes with a cross-entropy loss for classification. To ensure fairness, all methods were trained with the same feature extractor (e.g., ResNet18) with consistent hyperparameters, learning rate, batch size, and optimizer. We fixed the classification branch and replaced the divergence learning branch. To show the effectiveness of Bregman divergence, we replaced the similarity functions (cosine similarity, Euclidean distance) with the Bregman divergence in the SOTA methods to perform metrics comparison.

\subsubsection{Quantitative Results}
To demonstrate the promise of the proposed method, we performed image classification tasks on five datasets and compared the prediction performance of our approach with other SOTA methods. The results are shown in \textbf{Table I}. Our approach significantly outperforms the other SOTA methods with higher overall accuracy and AUC. The results of the iChallenge-PM dataset indicate that all methods can achieve over 95\% accuracy and AUC, demonstrating the feasibility of identifying pathological myopia from color fundus images. However, the performance drops for all methods on the iChallenge-AMD dataset due to insufficient annotated samples. For the CUB200 dataset, even though the dataset contains a large number of samples, the classification performance achieved similar results as the iChallenge-AMD dataset since the bird patterns are harder to detect, making the classification task more difficult challenging. Our approach outperforms other SOTA methods, showing the effectiveness of learned Bregman divergence for image recognition tasks. 
\begin{table*}
\centering
\caption{Competing for SOTA deep metric learning and contrastive learning methods on five selected datasets (UNIT: \%). ResNet18 is employed as a network encoder for feature extraction}
\sisetup{table-format=-1.3, table-number-alignment=center}
\begin{tabular}{r*{5}{SSc}SS}
\hline
&\multicolumn{2}{c}{iChallenge-PM} &&\multicolumn{2}{c}{iChallenge-AMD}&&\multicolumn{2}{c}{CUB200}&&\multicolumn{2}{c}{AFHQ}&&\multicolumn{2}{c}{Oxford-III Pet}&&\\
\cmidrule(l){2-3} \cmidrule(l){5-6} \cmidrule(l){8-9}\cmidrule(l){11-12}\cmidrule(l){14-16}
SOTAs & {Accuracy} & {AUC} & & {Accuracy} & {AUC} & & {Accuracy} & {AUC} & &{Accuracy}& {AUC} & & {Accuracy} & {AUC}\\
\hline
Baseline \cite{he2016deep} & 95.45 & 96.01 && 84.14 & 76.51 && 71.02 & 69.14 && 76.42 & 75.54 && 77.50 & 78.24\\
Siamese  \cite{koch2015siamese} & 95.12 & 97.21 && 78.14 & 69.45 && 77.14 & 73.45 && 80.25 & 82.64 && 85.12 & 84.54\\
Triplet  \cite{hoffer2015deep} & 95.12 & 97.21 && 80.18 & 70.28 && 80.14 & 75.65 && 83.36 & 84.45 && 85.50 & 84.47\\
N-pair \cite{sohn2016improved} & 96.45 & 94.38 && 85.12 & 74.54 && 82.14 & 79.45 && 78.98 & 81.25 && 81.65 & 80.22\\
SupCon \cite{khosla2020supervised} & 98.22 & 98.06 && 85.64 & 73.24 && 81.45 & 78.46\ && 82.87 & 79.69 && 86.20 & 82.84\\
GHM \cite{li2023learning} & 95.24 & 95.36 && 82.47 & 72.58 && 79.45 & 77.41 && 81.54 & 79.65 && 82.41 & 83.10\\
PDBL \cite{siahkamari2020learning} & 98.57 & \textbf{98.42} && 85.04 & 78.69 && 80.47 & 80.14\ && 84.12 & 84.74 && 85.50 & 85.10\\
DeepDiv \cite{cilingir2020deep} & 97.25 & 98.05 && 86.51 & 73.65 && \textbf{83.47} & 80.87 && 82.05 & 81.25 && 81.10 & 78.65\\
\textbf{Ours} &\textbf{99.12} & 98.14 && \textbf{87.45} & \textbf{80.17} && 82.53 & \textbf{82.49} && \textbf{85.45} & \textbf{86.03} && \textbf{88.20} & \textbf{89.35}\\
\hline
\end{tabular}
\end{table*}

\subsubsection{Model Generalizability}
To prove the generalizability of our proposed method, we first compared our approach with other SOTA methods on the AFHQ dataset and saved each pre-trained model. Since AFHQ and Oxford-III Pet contain cat and dog classes, we employed Oxford-III Pet as an independent external dataset and evaluated model generalizability using the pre-trained models. The results are shown in the last two columns of \textbf{Table I}. Our approach achieved overall classification performance with more precise accuracy and AUC on internal validations using the AFHQ dataset and external validation using the Oxford-III Pet dataset. In this way, we presented the generalization capabilities of learned empirical Bregman divergence of our proposed method for image classification.

\subsection{Ablation Study}
\subsubsection{Impact of Divergence Learning Loss}
Our learning objective is a linear combination of two loss functions, i.e., $ \mathcal{L}_{ce}$ and $\mathcal{L}_{div}$. Here, we analyzed the importance of the divergence loss $\mathcal{L}_{div}$ by training our approach with different $\gamma$ using the iChallenge-AMD dataset, in which $\gamma$ indicates a weighting factor of $\mathcal{L}_{div}$. The results are demonstrated in $\textbf{Figure 3}$. We found that when $\gamma=0.0$, the network is equivalent to a supervised baseline method with 84.14\% on accuracy and 76.51\% on AUC. As $\gamma$ increases, the performance improves and reaches the best performance when $\gamma=1.0$. This shows that the classification and divergence learning branches contribute equally to diagnosing age-related macular degeneration (AMD). 
\begin{figure}
    \centering
    \includegraphics[width=7.5cm]{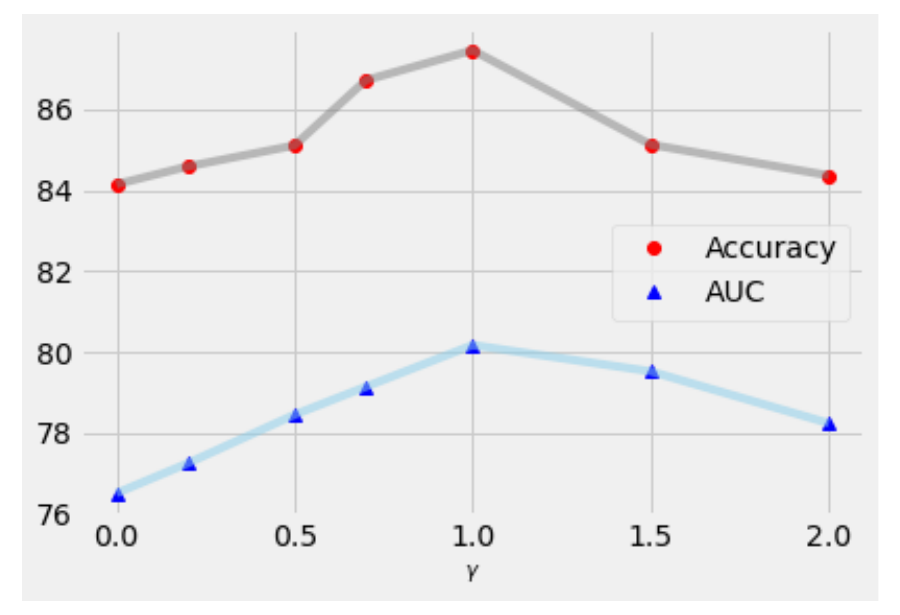}
    \caption{Training model on the iChallenge-AMD using different $\gamma$ of $\mathcal{L}^{*}$ in \textbf{Algorithm 1}. We achieved the best accuracy and AUC when $\gamma=1.0$.}
\end{figure}

\subsubsection{Quality Representations}
To verify the effectiveness of the learned feature representation of our approach, we use t-SNE to represent the last fully connected layer after CNN. As shown in \textbf{Figure 3}, we compared our approach with other SOTA metric learning and divergence learning methods on the 1000 testing AFHQ dataset. It is observed that our approach demonstrates a more precise decision boundary between the two classes. These results further show that learning the empirical Bregman divergence provides a better solution to capture the discriminative patterns. 
 \begin{figure}
    \centering
    \includegraphics[width=9.0cm]{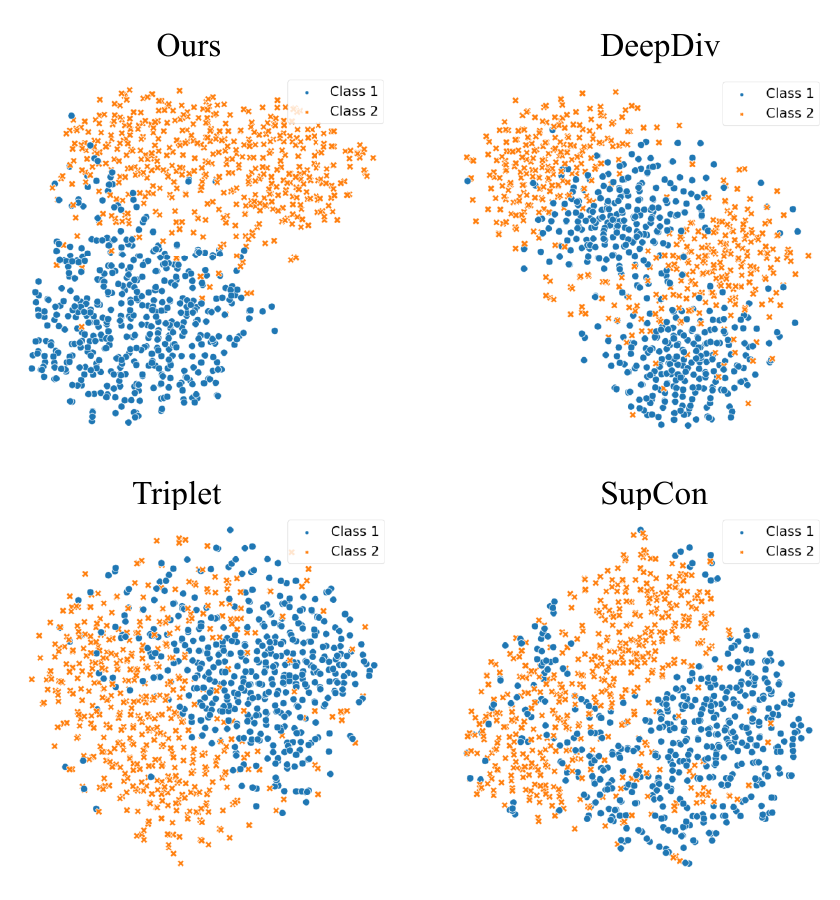}
    \caption{t-SNE visualization of learned embeddings from ResNet18 on the AFHQ dataset. Our approach precisely captures the decision boundary for separating two classes.}
\end{figure}
 \begin{figure*}
    \centering
     \includegraphics[width=0.85\textwidth]{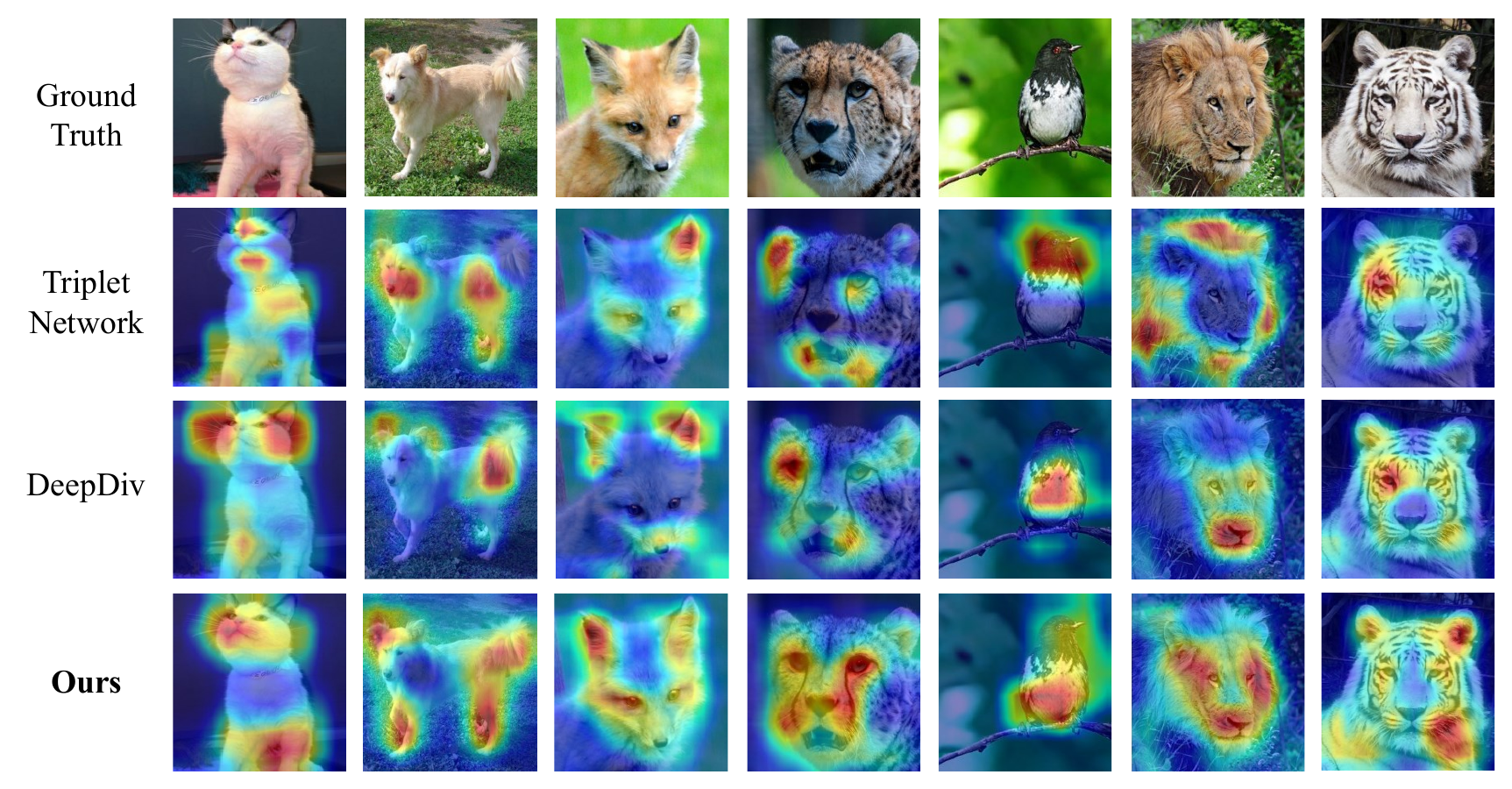}
    \caption{Feature visualization with different SOTA methods on three datasets (AFHQ, Oxford-III pet, and CUB200) using GradCAM. The more discriminative patterns of images indicate the high attention scores of the heatmap. Our approach captures more patterns than other methods, resulting in better classification performance.}
\end{figure*}
\subsubsection{Feature Visualization}
We compared our approach to Triplet Network \cite{hoffer2015deep} and Deep Divergence Learning \cite{cilingir2020deep} by showing the feature attention maps of the last ResNet18 block (\textbf{Figure 4}). We randomly selected seven input images from AFHQ, Oxford-III pet, and CUB200 datasets. We applied Grad-CAM \cite{selvaraju2017grad} to localize the discriminative patterns by pointwise multiplying the attention map with backpropagation corresponding to image classification. This visualization suggests the attention to various image patterns in each model for classification. Compared to other SOTA methods, our approach learns the empirical Bregman divergence that can help the network focus on the correction positions of images in terms of learning a more robust feature representation for classification.

\subsubsection{Metrics Comparison}
This section shows the advantage of learned Bregman divergence for capturing complex similarity using synthetic examples where existing approaches would fail. Assuming the relationship between two embedding presents a complex distribution, i.e., a random nonlinear correlation. We split the synthetic dataset into a supporting set and a query set, in which we train a Siamese network using the supporting set and apply the pre-trained network to match the query set. \textbf{Figure 5} shows each sample input (pixel) is highlighted if it matches the query input. We compared empirical Bregman divergence with other fixed distance functions in this setting, including cosine similarity and KL-divergence. As we can see, the learned Bregman divergence captures the best representation of nonlinear similarity among two embeddings. At the same time, other distances are not discriminative enough to capture the basic patterns under this complex distribution.
 \begin{figure}
    \centering
      \includegraphics[width=9.0cm]{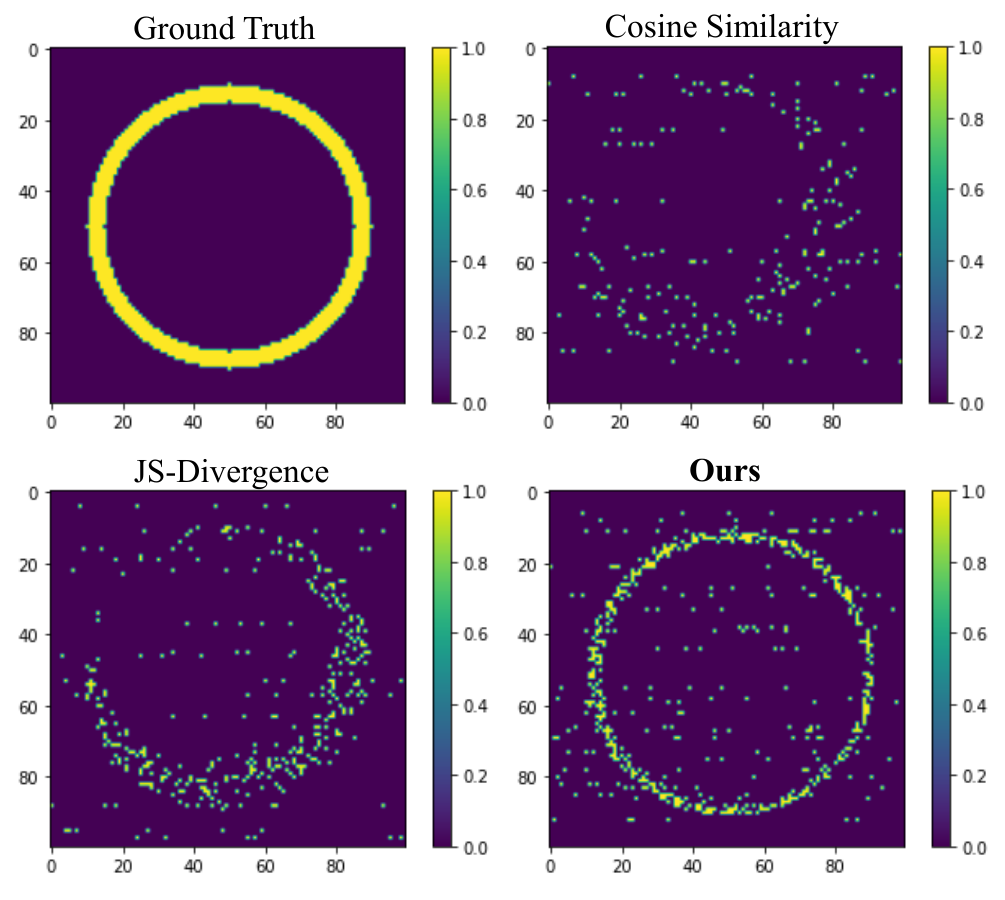}
    \caption{
The synthetic example showing learned Bregman divergence can capture more complex distribution on a matching. Each color pixel indicates the correct matched case from the supporting and query sets. 
}
\end{figure}

\subsubsection{Effects of the Number of m-Softplus Regressions}
Our proposed method contains $m$-Softplus regression of the GNM layer for parameterizing the convex function $\phi$ of the Bregman divergence. Here, we study the effects of different $m$ on classification performance. To assess it, we train our model with different $m$ and then compare the performance on the datasets of the iChallenge-AMD, CUB200, and AFHQ. The results are shown in $\textbf{Figure 6}$. We can see that performance for all datasets increases until $m=150$, then drops. We also observed a similar phenomenon in \cite{rezaei2021deep}. 
 \begin{figure}[ht]
    \centering
    \includegraphics[width=7.5cm]{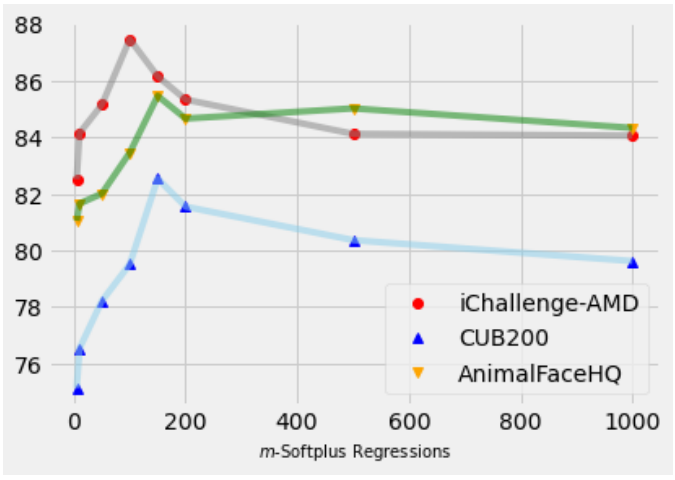}
    \caption{The importance of increasing Softplus regressions (m) on accuracy (\%): We train our model with different $m$ on the iChallenge-AMD, CUB200, and AFHQ datasets. The performance reaches best when $m=150$, then drops down because of over-parameterization. 
}
\end{figure}

\section{Discussion and Future Work}
Learning a representative distance is vital in visual representation for enhancing machine vision and pattern recognition. The learned distance representation can be further applied to various downstream tasks, including classification, clustering, and object detection. With the advances in deep learning techniques, deep metric learning has been widely used in the visual representation, and machine intelligence community \cite{meyer2019importance,mees2017metric,li2015weakly}. Besides promising evidence from previous studies \cite{koch2015siamese, hoffer2015deep, khosla2020supervised}, classic deep metric learning employed fixed distance metrics as the similarity function during the training, resulting in ignoring natural data distribution. Across probability theory and information science, the Bregman divergence uses a strictly convex function to represent a general distance metric, which provides a potential solution to address the challenge of arbitrary distance selection. This work first proves the equivalent relationship between a general metric learning loss and the Bregman divergence. We then present a novel approach to learn the empirical Bregman divergence by parameterizing a convex function between two feature embeddings in a deep metric learning style. Unlike previous works, our approach directly learns an optimal distance representation from data, showing practical advances for complex sample distributions. Compared to other SOTA methods, our approach consistently achieves promising results on five public datasets, which shows the supervisor of the learned distance representation. In addition to performance evaluation with other SOTA methods, extensive ablation studies are provided to further prove our approach's effectiveness.

Although our approach outperforms other SOTA methods, it still comes with limitations. First, we only study learning a Bregman divergence in a supervised manner, which relies on a more significant number of annotated training samples and requires expensive human effort. Secondly, our approach employs a GNM layer, which may be computationally costly if $m$ is large. In the future, we will investigate learning empirical Bregman divergence in an unsupervised or self-supervised learning style and study a more efficient alternative approach. 

\section{Acknowledgement}
Professor Anca Ralescu would like to thank Professor Shun'ichi Amari who first mentioned the Bregman divergence to her.

\bibliographystyle{IEEEtran}
\bibliography{ref.bib}
\end{document}